%% file: arxiv.tex
\documentclass[10pt]{article}

\usepackage{amsmath,amsthm,verbatim,amssymb,amsfonts,amscd, graphicx, enumitem}
\usepackage{graphics}
\usepackage{centernot}
\usepackage{authblk}
\theoremstyle{plain}
\newtheorem{theorem}{Theorem}

\usepackage[colorlinks,linkcolor=blue,citecolor=blue]{hyperref}

\usepackage{hyperref}
\usepackage{url}
\usepackage{natbib} 
\usepackage{booktabs}       
\usepackage{amsfonts}       
\usepackage{microtype}      
\usepackage{xcolor}

\usepackage{graphicx}  
\usepackage{subcaption}  
\usepackage{amssymb,amsmath,amsthm}
\usepackage{tcolorbox}
\usepackage{url}
\usepackage{listings}
\usepackage{geometry}                
\usepackage{soul}
\usepackage{amsbsy}
\usepackage{multirow}
\usepackage{longtable}
\usepackage{makecell}
\usepackage{array}     
\usepackage{hyperref}
\usepackage{parskip}
\usepackage[flushleft]{threeparttable}
\usepackage[ruled,vlined]{algorithm2e}
\usepackage{tabularx}

\title{Information-Theoretic Criteria for Knowledge Distillation in Multimodal Learning}

\author{Rongrong Xie$^1$, Yizhou Xu$^2$, Guido Sanguinetti$^1$\\
$^1$\textit{Scuola Internazionale Superiore di Studi Avanzati (SISSA)}\\
$^2$\textit{École Polytechnique Fédérale de Lausanne (EPFL)}
}

\begin{document}

\maketitle

\begin{abstract}
The rapid increase in multimodal data availability has sparked significant interest in cross-modal knowledge distillation (KD) techniques, where richer "teacher" modalities transfer information to weaker "student" modalities during model training to improve performance. However, despite successes across various applications, cross-modal KD does not always result in improved outcomes, primarily due to a limited theoretical understanding that could inform practice. To address this gap, we introduce the Cross-modal Complementarity Hypothesis (CCH): we propose that cross-modal KD is effective when the mutual information between teacher and student representations exceeds the mutual information between the student representation and the labels. We theoretically validate the CCH in a joint Gaussian model and further confirm it empirically across diverse multimodal datasets, including image, text, video, audio, and cancer-related omics data. Our study establishes a novel theoretical framework for understanding cross-modal KD and offers practical guidelines based on the CCH criterion to select optimal teacher modalities for improving the performance of weaker modalities.
\end{abstract}

\input{updates/Introduction}
\input{updates/RelatedWorks}
\input{updates/Method}
\input{updates/Experiment}

\input{updates/Conclusion}

\newpage
\section*{Reproducibility Statement}
The source code underpinning the experiments and analyses presented in this manuscript has been made accessible via an anonymized GitHub repository:
\begin{center}
  \url{https://anonymous.4open.science/r/test-111/}.
\end{center}
All experiment details are presented in Appendices \ref{app:synthetic}-\ref{app:mi_estimators}.
\bibliographystyle{plainnat}
\bibliography{references} 

\newpage
\appendix
\input{updates/Appendix}

\end{document}

%% file: updates/Introduction.tex
\section{Introduction}

Knowledge distillation (KD) transfers knowledge from a well-performing "teacher" model  to a smaller, simpler "student" model in order to reduce computational costs at prediction time\citep{camilli2023fundamental, maillard2024bayes, gou2021knowledge, choi2023understanding, cheng2020explaining, huang2022knowledge, tang2020understanding}. In standard KD, teacher and student networks have access to the same type of input data \citep{mishra2017apprentice}; however, with the increasing availability of multimodal data, cross-modal KD has become increasingly popular \citep{liu2023multimodal}. 

Cross-modal KD enables a student network, typically operating on a less informative modality, to benefit from richer representations provided by a teacher network trained on a more informative modality \citep{gupta2016cross, dai2021learning, ahmad2024multi, nair2024let}. Such methods are particularly valuable in scenarios where richer auxiliary modalities, such as video, audio, or text, are available during training, but only a single limited modality is accessible during testing \citep{du2021improving, kim2024privacy, zhao2024crkd, radevski2022students}. Another prominent example is medical diagnostics, where costly procedures like tissue biopsies or genomic sequencing may be available for a subset of patients, while more standard analyses are available for much larger cohorts. Cross-modal KD in principle enables a teacher trained with these privileged datasets to effectively guide a student model that relies solely on routine inputs \citep{jiang2021unpaired, zhang2023multi}.

While attractive in principle, the theoretical foundations of cross-modal KD are still not well understood, and, alongside success stories, there are also reports of instances where cross-modal KD fails to improve or even degrades student performance \citep{croitoru2021teachtext, lee2023decomposed}. Previous research primarily attributes these negative effects to the modality gap, differences between modalities that obstruct knowledge transfer and result in misaligned supervisory signals \citep{yuzhe2024vexkd, ref:modality_gap1}. Various approaches have aimed to mitigate these issues through complex fusion strategies or bespoke loss functions \citep{thoker2019cross, wang2023learnable, bano2024fedcmd, ref:li2024correlation}, but the general applicability of these solutions remains unclear. 

Theoretical studies on cross-modal KD have so far been limited. \citet{ref:LUPI} introduced "privileged information," a theoretical concept demonstrating that extra training-only data can improve model robustness. Building on this idea, \citet{ref:lopez2015unifying} developed the "generalized distillation" framework, demonstrating that distilling knowledge from privileged information reduces the student's sample complexity and accelerates training convergence. More recently, \citet{ref:MFH} empirically showed that the effectiveness of cross-modal KD significantly depends on the degree of label-relevant information shared between teacher and student modalities. Despite these insights, existing research has yet to determine a quantifiable criterion for successful cross-modal KD.

To address this gap, we introduce the Cross-modal Complementarity Hypothesis (CCH), a simple criterion based on mutual information which enables the user to {\it a priori} decide on whether cross-modal KD can be successful. We prove the validity of the CCH criterion in simplified scenarios, and test it empirically across a number of data sets.
The primary contributions of this paper are as follows:
\begin{itemize}
    \item Introduction of the Cross-modal Complementarity Hypothesis (CCH), proposing conditions under which cross-modal KD yields performance gains based on mutual information criteria.
    \item Proof of the validaty of the CCH criterion in the latent (jointly) Gaussian case.
    \item Extensive empirical validation through diverse experiments on multimodal datasets, including image, text, video, audio, and cancer-related omics data, confirming the practical utility of the proposed CCH criterion and providing actionable guidance for selecting effective teacher modalities.
\end{itemize}

%% file: updates/RelatedWorks.tex
\section{Related work}

\subsection{Unimodal KD}
KD is a powerful technique for transferring the detailed class information learned by a large teacher model to a smaller student model. Formally, consider a supervised \(K\)-class classification problem where both teacher and student classifiers receive the same input modality \(X\) and produce logits over the \(K\) classes. Let \(z_{\theta_1}(X)\) and \(z_{\theta_2}(X)\) denote the pre‐softmax logits of the teacher and student, respectively. Given a temperature \(T\), we define the softened outputs
\[
f_{\theta_i}(X;T) = \text{softmax}\!\bigl(z_{\theta_i}(X)/T\bigr).
\]
The student is trained to minimize a weighted combination of the cross‐entropy loss with respect to the ground‐truth labels \(Y\) and the distillation loss:
\begin{equation}
  \label{eq:unimodal_kd}
  \mathcal{L}
  = (1-\lambda)\,\mathrm{CE}\bigl(Y,\,f_{\theta_2}(X;1)\bigr)
  + \lambda\,T^2\,\mathrm{KL}\bigl(f_{\theta_1}(X;T)\,\|\,f_{\theta_2}(X;T)\bigr),
\end{equation}
where \(\lambda\in[0,1]\) balances learning directly from labels with learning from the teacher's predictions. The factor \(T^2\) compensates for smaller gradients at higher temperatures, and the softened teacher outputs \(f_{\theta_1}(X;T)\) convey richer inter‐class relationships than one‐hot labels alone \citep{hinton2015distilling}.

\subsection{Cross‐Modal KD}
Cross‐modal KD generalizes the unimodal framework to heterogeneous modalities, allowing a teacher with access to a stronger modality to guide a student with a weaker one. Consider two distinct modalities, denoted by \(X_1\) and \(X_2\), processed by the teacher and student models, respectively. The training objective extends Eq.~\eqref{eq:unimodal_kd} by appropriately substituting these distinct inputs \citep{liu2021deep}:
\begin{equation}
    \mathcal{L}
= (1-\lambda)\,\mathrm{CE}\bigl(Y,\,f_{\theta_2}(X_2;1)\bigr)
+ \lambda\,T^2\,\mathrm{KL}\bigl(f_{\theta_1}(X_1;T)\,\|\,f_{\theta_2}(X_2;T)\bigr).
\label{eq:crosskd_classification}
\end{equation}

\paragraph{Modality gaps}
Cross-modal KD encounters substantial obstacles due to the inherent modality gap between the teacher and student data representations. These disparities arise because modalities like images, text, and audio capture and encode information through fundamentally distinct physical processes and mathematical formalisms  \citep{ref:hu2023teacher, ref:sarkar2024xkd, ref:wang2025cross}. Previous research indicates that modality gaps lead to both modality imbalance—the disparity in predictive power across modalities—and soft label misalignment—where the teacher's outputs do not align with the student’s feature space. Consequently, these issues severely hinder effective knowledge transfer, thereby diminishing the efficacy of distillation \citep{ref:modality_gap1}. To mitigate these challenges, several studies have framed cross-modal KD as an information-maximization problem, proposing that effective transfer is achieved by maximizing the mutual information between the teacher’s and student’s representations or outputs \citep{ref:ahn2019variational, ref:chen2021distilling, ref:shrivastava2023estimating, ref:xia2023achieving, ref:shi2024multi, ref:li2024correlation}.

\paragraph{Theoretical foundations}
\citet{ref:LUPI} introduced the concept of "privileged information" as data available only during training. This provides a theoretical reason why additional inputs—often from a different modality—can improve model robustness. This idea naturally applies to cross-modal transfer, where the teacher's modality acts as privileged information for the student. Building on this idea, later work \citet{ref:lopez2015unifying} unified knowledge distillation with the privileged information framework, providing both theoretical and causal insights. Recent hypotheses further suggest that the success of cross-modal KD largely depends on the proportion of label-relevant information shared between teacher and student modalities \citep{ref:MFH}. Another related hypothesis proposes that domain gaps mainly affect student performance through errors in non-target classes. Theoretical analyses based on VC theory show that reducing divergence in these off-target predictions improves student performance \citep{ref:chen2024non}. Despite these advances, no previous work has explicitly defined conditions based on mutual information to determine when cross-modal KD is feasible.

%% file: updates/Method.tex
\section{The Cross-modal Complementarity Hypothesis}


We study cross-modal KD in settings where the teacher and student models access modalities of unequal predictive power.  Let \(X_1\) and \(X_2\) denote two data modalities whose intrinsic capacities differ, and let \(Y\) be the ground‐truth label.  Concretely, we assume \(X_1\) to be the inputs to the teacher network, i.e. the data associated with the strong modality which is highly predictive of the output labels,  while \(X_2\) is the weak modality supplied to the student. The primary goal of cross-modal KD in this context is to transfer the label-relevant representations from the strong modality $X_1$ to the weak modality $X_2$, thereby augmenting the student's performance. This raises a fundamental question: under what conditions can a teacher operating on a strong modality effectively compensate for the insufficiencies of a weak modality?

Denote $H_1,H_2$ to be the represenation of $X_1,X_2$. 
Our intuition is that if the mutual information between $H_1$ and $H_2$, denoted by $I(H_1; H_2)$, exceeds the mutual information between $H_2$ and $Y$, denoted by $I(H_2; Y)$, the first term in contains more information than the second term, and thus the teacher modality $X_1$ can provide the complementary, label-relevant information that $X_2$ lacks. Also, a large $I(H_1; H_2)$ indicates substantial overlap between the modalities, suggesting that the student is capable of interpreting the teacher’s guidance. This condition ensures that the teacher's knowledge is sufficiently aligned with the student's domain to improve prediction accuracy through distillation.

We thus propose the following \emph{Cross-modal Complementarity Hypothesis}:

\begin{quote}
\textbf{Cross-modal Complementarity Hypothesis (CCH):} For cross-modal knowledge distillation, if 
\[
I(H_1; H_2) > I(H_2; Y),
\]
then the teacher modality can supply compensatory information, leading to improved student performance, where $H_1,H_2$ are teacher and student representations,
\end{quote}



In the rest of this section, we support mathematically this intuition in a simple but tractable case.

Assume that the dataset $\{(x_{1i}, x_{2i}, y_i)\}_{i=1}^n$ is jointly Gaussian distributed:
\begin{equation}
\left\{\left(\begin{array}{c}
     x_{1i}\\
     x_{2i}\\
     y_i
\end{array}\right)\right\}_{i=1}^n\overset{i.i.d.}{\sim}\mathcal{N}\left(0,\left(\begin{array}{ccc}
     \Sigma_{11} & \Sigma_{12} & \Sigma_{13}\\
     \Sigma_{12}^T & \Sigma_{22} & \Sigma_{23}\\
     \Sigma_{13}^T & \Sigma_{23}^T & \Sigma_{33}
\end{array}\right)\right),
\label{eq:linear_model}
\end{equation}
where $x_{1i},x_{2i}\in\mathbb{R}^p$ and $y\in\mathbb{R}$. We consider the limit $n,p\to\infty$ with $\frac{n}{p}\to\kappa$ and the operation norm of each $\Sigma_{ij}\ (1\leq i,j\leq 3)$ is bounded by a constant.

The associated learning task is a multi-modal (linear) regression problem with data 
$\mathcal{D} = \{x_{1i}, x_{2i},y_i\}_{i=1}^n$. 
The outputs of the teacher and student networks for the $i$-th sample are $w_1^T x_{1i}$ and $w_2^T x_{2i}$, respectively, where $w_1$ and $w_2$ are the trainable parameters. The cross-modal  objective for training the student is given by
\begin{equation}
\hat{w}:=\arg\min_{w_2}\sum_{i=1}^n \Bigl\|y_i - w_2^T x_{2i}\Bigr\|^2 
\;+\; \lambda \sum_{i=1}^n \Bigl\|w_2^T x_{2i} - w_1^T x_{1i}\Bigr\|^2,
\label{eq:old_loss}
\end{equation}
where the first term measures the discrepancy between the ground-truth label and the student’s predictions, and the second term, weighted by $\lambda$, enforces alignment between teacher and student outputs.

The excess risk is given by
\begin{equation}
R(\lambda,w_1):=\mathbb{E}_{x_1,x_2,y}[(y-(\hat{w})^Tx_2)^2]-\sigma^2,
\end{equation}
which is regarded as a function of the teacher weights $w_1$ (with bounded norm) and the regularization strength $\lambda$. We then define $R_0:=R(0,w_1)$ to be the baseline performance, where the teacher is absent and obviously $R_0$ does not depend on $w_1$. Then we have the following theorem.
\begin{theorem}
Assume that $\kappa>1$ and $w_1^T\Sigma_{11}w_1\leq\Sigma_{33}$, $w_1^T\Sigma_{13}\geq0$. Suppose that $I(w_1^Tx_1,(w^*)^Tx_2)>I((w^*)^Tx_2,y)$, where $w^*:=\Sigma_{22}^{-1}\Sigma_{23}$ is the optimal student weight, then we have
\begin{equation}
R(\lambda,w_1)<R_0
\end{equation}
asymptotically for small $\lambda$.
\label{theo:main}
\end{theorem}
Note that $w_1^T\Sigma_{11}w_1\leq\Sigma_{33}$, $w_1^T\Sigma_{13}\geq0$ are mild assumptions that the teacher weights should not be too large or too misleading. Notably the optimal teacher weight $\Sigma_{11}^{-1}\Sigma_{13}$ satisfies these two assumptions.

Theorem \ref{theo:main} suggests that knowledge distillation is beneficial when the mutual information between teacher and student representations are larger than the mutual information between student representations and the teacher. It is proved in Appendix \ref{app:theory}.

%% file: updates/Experiment.tex
\section{Experiments}

To validate the proposed Cross-modal Complementarity Hypothesis (CCH), we conducted extensive experiments across various datasets, including synthetic data, image, text, video, audio, and cancer-related omics datasets. To systematically assess how mutual information influences the effectiveness of cross-modal KD, the teacher and student networks were intentionally configured to have identical architectures in all experiments. This design choice facilitates a clear and unbiased comparison, isolating mutual information as the primary variable affecting knowledge transfer effectiveness.

\subsection{Synthetic data}\label{sec:synthetic_data}

We generate synthetic data for a regression task by drawing $n$ i.i.d.\ samples from a zero-mean multivariate Gaussian model (cf.\ Eq.~\ref{eq:linear_model}) over a teacher modality $X_1\!\in\!\mathbb{R}^{n\times p}$, a student modality $X_2\!\in\!\mathbb{R}^{n\times p}$, and a scalar target $Y\!\in\!\mathbb{R}^{n}$. To enable controlled analyses, we specialize the Gaussian model by parameterizing all cross-covariances as scalar multiples of the identity. Specifically,
\[
\Sigma_{12}=\sigma_{12} I_p,\quad
\Sigma_{13}=\sigma_{13} I_p,\quad
\Sigma_{23}=\sigma_{23} I_p,\quad
\mathrm{Var}(Y)=1,
\]
where each $\sigma_{ij}\in(-1,1)$ governs the corresponding pairwise correlation. Under this parameterization,
\begin{equation}
\left\{\begin{pmatrix}
x_{1i}\\[2pt]
x_{2i}\\[2pt]
y_i
\end{pmatrix}\right\}_{i=1}^n
\sim \mathcal{N}\!\left(
0,\;
\begin{pmatrix}
I_p & \sigma_{12} I_p & \sigma_{13}\,\mathbf{1}_p \\
\sigma_{12} I_p & I_p & \sigma_{23}\,\mathbf{1}_p \\
\sigma_{13}\,\mathbf{1}_p^{\!\top} & \sigma_{23}\,\mathbf{1}_p^{\!\top} & 1
\end{pmatrix}
\right),
\label{eq:generate_synthetic}
\end{equation}
so that $I(X_1;X_2)$, $I(X_1;Y)$, and $I(X_2;Y)$ are monotone in $\sigma_{12}$, $\sigma_{13}$, and $\sigma_{23}$, respectively.

Unless otherwise stated, we set $n=10000$ and $p=100$. To study how student performance varies with cross-modal dependence, we fix the teacher–label correlation at $\sigma_{13}=0.9$ and the student–label correlation at $\sigma_{23}=0.4$, and vary $\sigma_{12}\in[0,0.7]$ to maintain positive semidefiniteness of the covariance.

Figure~\ref{fig:synthetic_nonlinear_lambda05} summarizes the results. Panel~\ref{fig:synthetic_nonlinear_lambda05_mse} reports the student test mean squared error (MSE) as $\sigma_{12}$ varies; each point averages ten random seeds. Panel~\ref{fig:synthetic_nonlinear_lambda05_mi} shows mutual information (MI) between learned representations: $I(H_1;H_2)$ for teacher $X_1$ and student $X_2$, and $I(H_2;Y)$ for the student and the label. We extract representations $H_1$ and $H_2$ from each network's feature extractor and estimate MI using the \texttt{latentmi} estimator~\citep{ref:gowri2024approximating}.

Empirically, knowledge distillation (KD) reduces MSE precisely when $I(H_1;H_2)>I(H_2;Y)$ and provides no benefit otherwise. This pattern supports the Cross-modal Complementarity Hypothesis (CCH): the teacher contributes complementary, label-relevant information when its representation shares more information with the student than the student shares with the label. Additional experiments across distillation weights $\lambda$ (Appendix~\ref{app:synthetic}) corroborate this trend.

\begin{figure}[ht]
    \centering
    \begin{subfigure}[b]{0.45\textwidth}
        \centering
        \includegraphics[width=\linewidth]{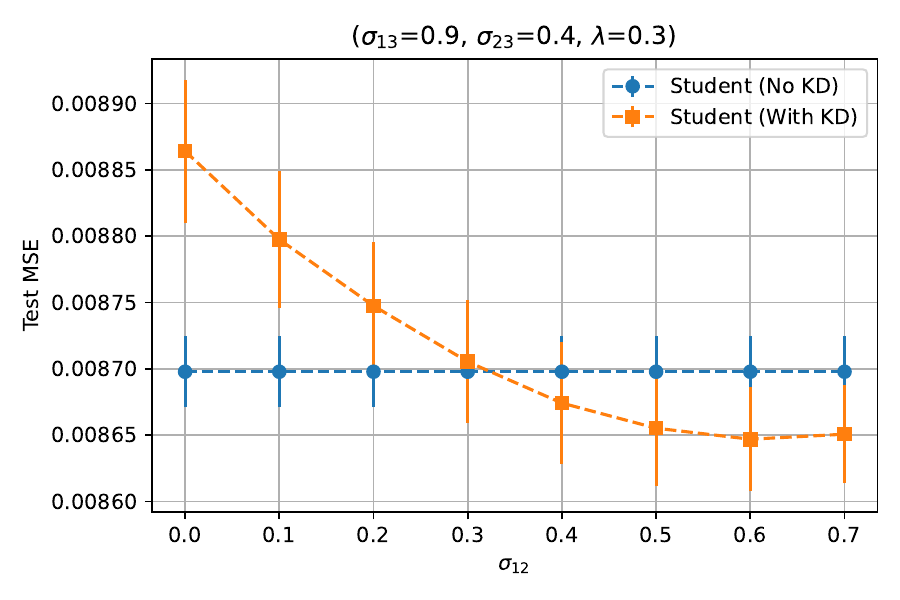}
        \caption{Student MSE vs.\ $\sigma_{12}$.}
        \label{fig:synthetic_nonlinear_lambda05_mse}
    \end{subfigure}
    \hfill
    \begin{subfigure}[b]{0.45\textwidth}
        \centering
        \includegraphics[width=\linewidth]{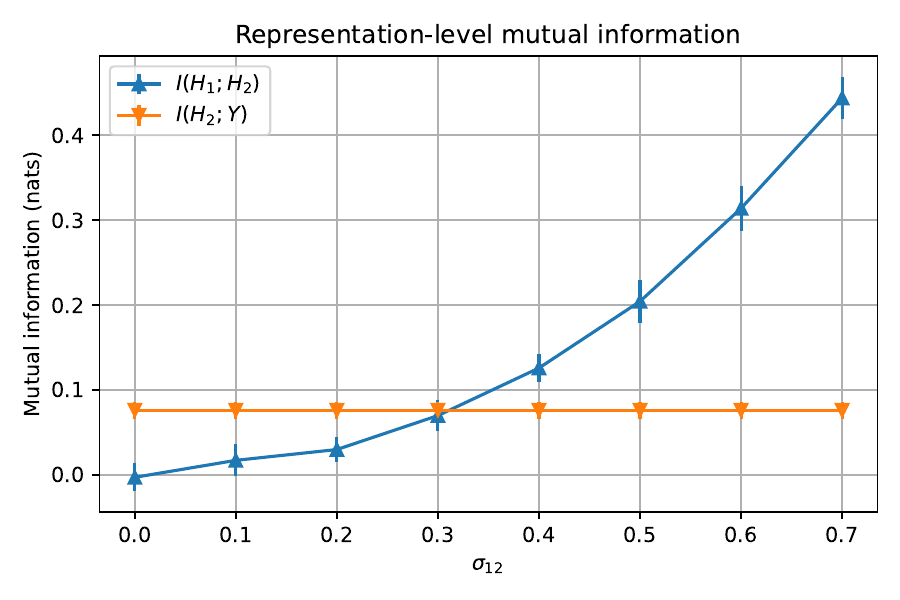}
        \caption{Representation MI vs.\ $\sigma_{12}$.}
        \label{fig:synthetic_nonlinear_lambda05_mi}
    \end{subfigure}
    \caption{Synthetic regression experiments. When $I(H_1;H_2)$ exceeds $I(H_2;Y)$, the KD-trained student achieves lower test MSE than a non-distilled student; otherwise, KD provides no improvement.}
    \label{fig:synthetic_nonlinear_lambda05}
\end{figure}

\subsection{Image data}

We conduct classification experiments on the MNIST \citep{lecun1998gradient} and MNIST-M datasets \citep{ganin2015unsupervised}. MNIST is a standard benchmark of $70{,}000$ handwritten digits (0–9), each a $28\times28$‐pixel grayscale image with a corresponding label. MNIST-M is derived by blending the binarized MNIST digits onto random natural-image patches from the BSDS500 dataset \citep{martin2001database}; thus, it represents a distinct modality while sharing identical labels with MNIST (see Figure~\ref{fig:mnist} in Appendix \ref{app:image}).

We treat MNIST as the \emph{teacher} modality and MNIST-M as the \emph{student} modality. First, we compute the mutual information between the teacher and student representations,
\(
I_{TS} \;=\; I\bigl(H_{\mathrm{MNIST}};H_{\mathrm{MNIST\text{-}M}}\bigr),
\)
and between the student represntations and labels,
\(
I_{SL} \;=\; I\bigl(H_{\mathrm{MNIST\text{-}M}};Y\bigr),
\)
using the \texttt{latentmi} estimator \citep{ref:gowri2024approximating}. We then follow the protocol in Algorithm~\ref{alg:image_protocol} (Appendix~\ref{app:image}). During distillation, we systematically vary $I_{TS}$ by applying Gaussian blur with standard deviation $\gamma$ to the teacher inputs, and assess whether the student’s accuracy gains correspond to the CCH condition $I_{TS}>I_{SL}$.

Figure \ref{fig:accuracy_mi_image} illustrates the impact of varying Gaussian blur intensity $\gamma$ on both the student’s test accuracy and the corresponding mutual information when the distillation temperature is at $T=3$ (see additional results in Appendix \ref{app:image}). Results are averaged over five independent runs. Panel~(a) compares the test accuracy of students trained with and without distillation; panel~(b) plots $I_{TS}$ and $I_{SL}$ as functions of~$\gamma$. We observe that whenever $I_{TS}>I_{SL}$, knowledge distillation improves accuracy relative to the baseline, in agreement with the CCH. For $\gamma \ge 2.5$, $I_{TS}$ falls below $I_{SL}$, leading to a collapse in the distilled student’s performance.

\begin{figure}[!t]
  \centering
  \begin{subfigure}[t]{0.48\textwidth}
    \centering
    \includegraphics[width=\linewidth]{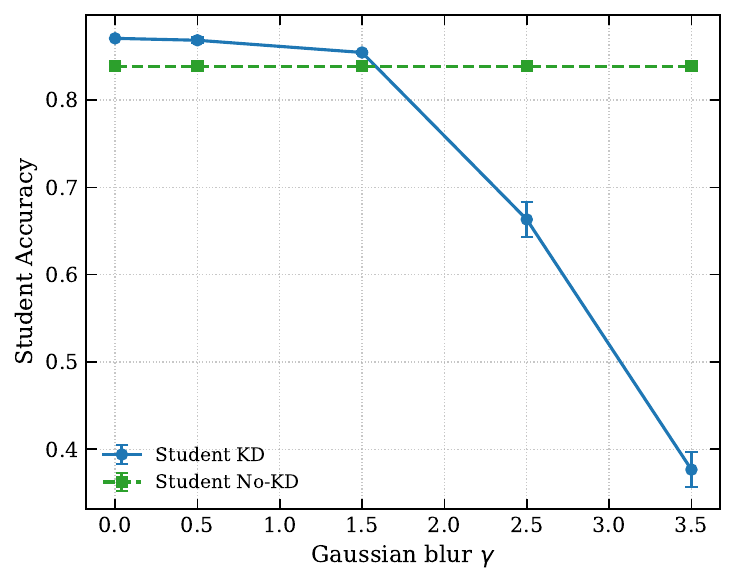}
    \caption{Student accuracy vs.\ Gaussian blur~$\gamma$.}
    \label{fig:accuracy}
  \end{subfigure}
  \hfill
  \begin{subfigure}[t]{0.48\textwidth}
    \centering
    \includegraphics[width=\linewidth]{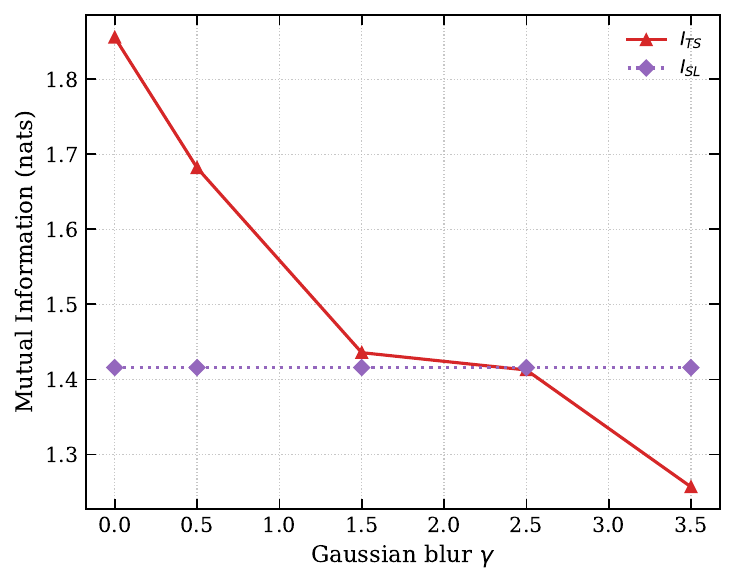}
    \caption{Mutual information vs.\ Gaussian blur~$\gamma$.}
    \label{fig:mi}
  \end{subfigure}
  \caption{Relationship between student accuracy and  mutual information under varying Gaussian blur. (\textbf{a}) Test accuracy of the MNIST--M student trained with (solid line) and without (dashed line) distillation as a function of Gaussian blur standard deviation $\gamma$ applied to MNIST teacher inputs. (\textbf{b}) Mutual information $I_{TS}=I(H_{\mathrm{MNIST}};H_{\mathrm{MNIST\text{-}M}})$ (red) and 
  $I_{SL}=I(H_{\mathrm{MNIST\text{-}M}};Y)$ (purple) 
  versus $\gamma$. Accuracy improvements align with the region where $I_{TS}>I_{SL}$. For reference, $I_{TL}=I(H_{\mathrm{MNIST}};Y)=2.0485$, and the teacher network attains a test accuracy of 0.981.}
  \label{fig:accuracy_mi_image}
\end{figure}

\begin{table}[!t]
  \centering
  \scriptsize
  \begin{threeparttable}
    \caption{Mutual‐information gap and student accuracy differ under varying blur and temperature.}
    \label{tab:accuracy_mi_diff_image}
    \begin{tabular}{@{} c  c  *{4}{c}  @{}}
      \toprule
      \multirow{2}{*}{$\gamma$}
        & \multirow{2}{*}{MI GAP (nats)}
        & \multicolumn{4}{c}{Student Acc.~Diff.\,($\pm$SE)} \\
      \cmidrule(l){3-6}
      &  & $T=1$ & $T=2$ & $T=3$ & $T=4$ \\
      \midrule
      0.0 &  0.4399   & $0.0010\pm0.0040$   & $0.0146\pm0.0035$   & $0.0318\pm0.0040$   & $0.0350\pm0.0046$   \\
      0.5 &  0.2662   & $0.0069\pm0.0054$   & $0.0152\pm0.0055$   & $0.0296\pm0.0031$   & $0.0353\pm0.0028$   \\
      1.5 &  0.0199   & $0.0002\pm0.0089$   & $0.0149\pm0.0034$   & $0.0156\pm0.0042$   & $0.0091\pm0.0051$   \\
      2.5 & $-0.0032$ & $-0.1190\pm0.0165$  & $-0.1627\pm0.0101$  & $-0.1757\pm0.0219$  & $-0.1516\pm0.0154$  \\
      3.5 & $-0.1590$ & $-0.2797\pm0.0126$  & $-0.4597\pm0.0041$  & $-0.4623\pm0.0209$  & $-0.4364\pm0.0137$  \\
      \bottomrule
    \end{tabular}  
  \end{threeparttable}
\end{table}

We further explore the effect of the distillation temperature $T\in\{1,2,3,4\}$ in Table~\ref{tab:accuracy_mi_diff_image}. Here, \emph{MI GAP} denotes $I_{TS}-I_{SL}$, and \emph{Student Acc.~Diff.} is the difference in test accuracy between the distilled and baseline students. SE denotes the standard error estimated from five independent runs. Across all blur levels and temperatures, the sign of the \emph{Student Acc.~Diff.} matches that of the \emph{MI GAP}, reinforcing the CCH. We remark the very non-linear behaviour of the student's accuracy w.r.t. the MI GAP; while the gain remains modest for positive MI GAP, as soon as the MI GAP changes sign we document a very large drop in student accuracy.

\subsection{CMU-MOSEI dataset}

We evaluate the CCH on the CMU Multimodal Opinion Sentiment and Emotion Intensity (CMU-MOSEI) dataset \citep{zadeh2018multimodal}. CMU-MOSEI is a large-scale benchmark for multimodal sentiment analysis comprising 23{,}453 annotated video segments with time-aligned text, vision, and audio streams drawn from 1{,}000 speakers across 250 topics.

\begin{table}[!t]
\centering
\setlength{\tabcolsep}{6pt}
\renewcommand{\arraystretch}{1.1}
\caption{Mutual information estimates between CMU-MOSEI modality representations and the label using three estimators (mean $\pm$~std over 50 runs).}
\label{tab:mi}
\resizebox{\linewidth}{!}{%
\begin{tabular}{lccccc}
\toprule
Estimator & $I(H_{\mathrm{text}};H_{\mathrm{vision}})$ & $I(H_{\mathrm{text}};H_{\mathrm{audio}})$ & $I(H_{\mathrm{text}};Y)$ & $I(H_{\mathrm{vision}};Y)$ & $I(H_{\mathrm{audio}};Y)$ \\
\midrule
\texttt{latentmi} & $1.3543\pm0.0052$ & $1.4160\pm0.0038$ & $0.4681\pm0.0090$ & $0.0816\pm0.0084$ & $0.1054\pm0.0088$ \\
\texttt{mine}     & $0.7955\pm0.0019$ & $1.1817\pm0.0023$ & $0.3202\pm0.0055$ & $0.0409\pm0.0026$ & $0.0631\pm0.0026$ \\
\texttt{ksg}      & $0.3788\pm0.0056$ & $0.6606\pm0.0056$ & $0.1628\pm0.0083$ & $0.0647\pm0.0014$ & $0.0934\pm0.0018$ \\
\bottomrule
\end{tabular}%
}
\end{table}

\begin{table}[!t]
\centering
\setlength{\tabcolsep}{6pt}
\renewcommand{\arraystretch}{1.15}
\caption{Student performance versus mutual information on CMU-MOSEI with text as teacher. The teacher achieves test accuracy $0.7190 \pm 0.0098$ and weighted F1 $0.7189 \pm 0.0098$; $I(H_{\mathrm{text}};Y)=0.4681 \pm 0.0090$. Mutual information is estimated with \texttt{latentmi}.}
\label{tab:mosei_kd}
\resizebox{\linewidth}{!}{%
\begin{tabular}{l cc cc cc}
\toprule
& \multirow{2}{*}{\centering $I(H_{\mathrm{teacher}};H_{\mathrm{student}})$}
& \multirow{2}{*}{\centering $I(H_{\mathrm{student}};Y)$}
& \multicolumn{2}{c}{Student Without KD}
& \multicolumn{2}{c}{Student With KD} \\
\cmidrule(lr){4-5}\cmidrule(lr){6-7}
& & & Acc & Weighted F1 & Acc & Weighted F1 \\
\midrule
\makecell[l]{Text (teacher)\\ Vision (student)}
& $1.3543 \pm 0.0052$ & $0.0816 \pm 0.0084$
& $0.6233 \pm 0.0027$ & $0.6204 \pm 0.0030$
& $0.6343 \pm 0.0013$ & $0.6315 \pm 0.0022$ \\
\hline
\makecell[l]{Text (teacher)\\ Audio (student)}
& $1.4160 \pm 0.0038$ & $0.1054 \pm 0.0088$
& $0.5937 \pm 0.0048$ & $0.5931 \pm 0.0043$
& $0.6167 \pm 0.0030$ & $0.6161 \pm 0.0031$ \\
\bottomrule
\end{tabular}%
}
\end{table}

\begin{table}[!t]
  \caption{Student weighted F1 versus mutual information on the CMU-MOSEI dataset under varying levels of Gaussian noise (text teacher, vision student).}
  \label{tab:noise_mosei_kd}
  \centering
  \small
  \begin{tabular}{lcccc}
    \toprule
    Noise level & $I(H_{\mathrm{teacher}};H_{\mathrm{student}})$ & $I(H_{\mathrm{student}};Y)$ & Student KD F1 & Student No-KD F1 \\
    \midrule
    0\%  & $1.3543\pm0.0052$ & $0.0816\pm0.0084$ & $0.6204\pm0.0030$ & $0.6315\pm0.0022$ \\
    20\% & $0.0034\pm0.0040$ & $0.0816\pm0.0084$ & $0.6204\pm0.0030$ & $0.6192\pm0.0062$ \\
    40\% & $-0.0007\pm0.0045$ & $0.0816\pm0.0084$ & $0.6204\pm0.0030$ & $0.6189\pm0.0039$ \\
    60\% & $-0.0056\pm0.0058$ & $0.0816\pm0.0084$ & $0.6204\pm0.0030$ & $0.6184\pm0.0022$ \\
    80\% & $-0.0060\pm0.0053$ & $0.0816\pm0.0084$ & $0.6204\pm0.0030$ & $0.6156\pm0.0033$ \\
    \bottomrule
  \end{tabular}
\end{table}

The task is binary sentiment classification. Following standard practice, we binarize the original integer sentiment scores into positive and negative labels. Each utterance is converted into synchronized, fixed-length sequences for all three modalities using a uniform preprocessing pipeline; full details are provided in Appendix~\ref{app:mosei}.

To operationalize the CCH, we estimate mutual information (MI) between (i) each pair of modality representations and (ii) each modality representation and the label. We employ three complementary estimators—\texttt{latentmi} \citep{ref:gowri2024approximating}, \texttt{mine} \citep{ref:belghazi2018mutual}, and \texttt{ksg} \citep{ref:ross2014mutual}—and average results over 50 independent runs (Appendix~\ref{app:mi_estimators}). As shown in Table~\ref{tab:mi}, absolute MI values vary by estimator, but the relative ordering is consistent.

The MI patterns in Table~\ref{tab:mi} identify text as the most predictive modality, since $I(H_{\mathrm{text}};Y)$ is largest. Accordingly, we designate text as the teacher and treat vision and audio as student modalities. As reported in Table~\ref{tab:mosei_kd}, KD yields significant gains over the no-KD baseline for both students. Moreover, Table~\ref{tab:mi} shows that $I(H_{\mathrm{text}};H_{\mathrm{vision}}) > I(H_{\mathrm{vision}};Y)$ and $I(H_{\mathrm{text}};H_{\mathrm{audio}}) > I(H_{\mathrm{audio}};Y)$, satisfying the CCH condition. Taken together, these observations support the CCH. The improvement is larger for audio, consistent with its greater MI gap $I_{\mathrm{TS}}-I_{\mathrm{SL}}$ (teacher–student vs. student–label MI of representations), suggesting a positive association between the gap magnitude and KD efficacy.

To further probe the CCH, we conduct a controlled degradation experiment on the text (teacher) $\rightarrow$vision (student) setting. We inject Gaussian noise into the teacher input to systematically reduce $I(H_{\mathrm{teacher}};H_{\mathrm{student}})$ while holding $I(H_{\mathrm{student}};Y)$ fixed. As predicted, the benefit of KD disappears once $I(H_{\mathrm{teacher}};H_{\mathrm{student}}) < I(H_{\mathrm{student}};Y)$ (Table~\ref{tab:noise_mosei_kd}).

\subsection{Cancer data}

We analyze three The Cancer Genome Atlas (TCGA) cohorts~\citep{colaprico2016tcgabiolinks}: breast invasive carcinoma (BRCA), pan-kidney (KIPAN), and liver hepatocellular carcinoma (LIHC). For each cohort, we consider three omics modalities—mRNA expression (mRNA), copy number variation (CNV), and reverse-phase protein arrays (RPPA)—and retain only cases with complete data across all three. The learning task is subtype classification; Table~\ref{tab:subtypes} in Appendix~\ref{app:cancer} reports class distributions. To reduce noise and dimensionality, we preprocess each modality independently and select the top 100 features from the original sets of 60{,}660 (mRNA), 60{,}623 (CNV), and 487 (RPPA) using the minimum-redundancy maximum-relevance (mRMR) criterion~\citep{ding2005minimum}.

\begin{table}[!t]
\centering
\caption{Student weighted F1 vs.\ mutual information on BRCA under varying Gaussian noise levels (\emph{teacher:} mRNA; \emph{student:} CNV). The teacher achieves test weighted F1 of $0.7459$ and $I(H_{\mathrm{teacher}};Y)=1.1081$. ``MI Gap'' denotes $I_{\mathrm{TS}}-I_{\mathrm{SL}}$; ``Student F1 Difference'' denotes (Student KD F1) $-$ (Student No-KD F1).}
\resizebox{0.98\textwidth}{!}{%
\begin{tabular}{ccccccc}
\toprule
Noise Level & $I(H_{\mathrm{teacher}};H_{\mathrm{student}})$ & $I(H_{\mathrm{student}};Y)$ & Student KD F1 & Student No-KD F1 & MI GAP & Student F1 Differ \\
\midrule
0\%   & 0.5005 & 0.2757 & 0.5038 & 0.4561 & 0.2248 & 0.0477 \\
20\% & 0.4554 & 0.2757 & 0.4917 & 0.4561 & 0.1797 & 0.0356 \\
40\% & 0.3687 & 0.2757 & 0.4953 & 0.4561 & 0.0930 & 0.0392 \\
60\% & 0.2147 & 0.2757 & 0.4276 & 0.4561 & -0.061 & -0.0285 \\
80\% & 0.1325 & 0.2757 & 0.4343 & 0.4561 & -0.1432 & -0.0218 \\
\bottomrule
\end{tabular}%
}
\label{tab:brca_mRNAcnv}
\end{table}

\begin{table}[!t]
\centering
\caption{Student weighted F1 vs.\ mutual information on KIPAN under varying Gaussian noise levels (\emph{teacher:} mRNA; \emph{student:} CNV). The teacher achieves test weighted F1 of $0.9516$ and $I(H_{\mathrm{teacher}};Y)=1.0458$.}
\resizebox{0.98\textwidth}{!}{%
\begin{tabular}{ccccccc}
\toprule
Noise Level & $I(H_{\mathrm{teacher}};H_{\mathrm{student}})$ & $I(H_{\mathrm{student}};Y)$ & Student KD F1 & Student No-KD F1 & MI GAP & Student F1 Differ \\
\midrule
0\%   & 0.7898 & 0.6994 & 0.8826 & 0.8667 & 0.0904 & 0.0159 \\
20\% & 0.7198 & 0.6994 & 0.8721 & 0.8667 & 0.0204 & 0.0054 \\
40\% & 0.6771 & 0.6994 & 0.8517 & 0.8667 & -0.0223 & -0.0150 \\
60\% & 0.6209 & 0.6994 & 0.8477 & 0.8667 & -0.0785 & -0.0190 \\
80\% & 0.6389 & 0.6994 & 0.8544 & 0.8667 & -0.0605 & -0.0123 \\
\bottomrule
\end{tabular}%
}
\label{tab:kipan_mRNAcnv}
\end{table}

\begin{table}[!t]
\centering
\caption{Student weighted F1 vs.\ mutual information on LIHC under varying Gaussian noise levels (\emph{teacher:} mRNA; \emph{student:} CNV). The teacher achieves test weighted F1 of $0.9430$ and $I(H_{\mathrm{teacher}};Y)=0.9055$.}
\resizebox{0.98\textwidth}{!}{%
\begin{tabular}{ccccccc}
\toprule
Noise Level & $I(H_{\mathrm{teacher}};H_{\mathrm{student}})$ & $I(H_{\mathrm{student}};Y)$ & Student KD F1 & Student No-KD F1 & MI GAP & Student F1 Differ \\
\midrule
0\%   & 0.0914 & 0.0781 & 0.5795 & 0.5548 & 0.0133 & 0.0247 \\
20\% & 0.0825 & 0.0781 & 0.5692 & 0.5548 & 0.0044 & 0.0144 \\
40\% & 0.0699 & 0.0781 & 0.5368 & 0.5548 & -0.0082 & -0.0180 \\
60\% & 0.0736 & 0.0781 & 0.5259 & 0.5548 & -0.0045 & -0.0289 \\
80\% & 0.0409 & 0.0781 & 0.5080 & 0.5548 & -0.0372 & -0.0468 \\
\bottomrule
\end{tabular}%
}
\label{tab:lihc_mRNAcnv}
\end{table}

We first set mRNA as the teacher and CNV as the student and estimate
\[
I_{\mathrm{TS}} \!=\! I\!\bigl(H_{\mathrm{mRNA}};H_{\mathrm{CNV}}\bigr),
\qquad
I_{\mathrm{SL}} \!=\! I\!\bigl(H_{\mathrm{CNV}};Y\bigr),
\]
using the \texttt{latentmi} estimator. To modulate $I_{\mathrm{TS}}$, we add zero-mean Gaussian noise to the teacher inputs. Tables~\ref{tab:brca_mRNAcnv}–\ref{tab:lihc_mRNAcnv} report student weighted F1 and mutual information as functions of the noise level (means over five runs). Across cohorts, whenever the MI Gap is positive ($I_{\mathrm{TS}} > I_{\mathrm{SL}}$), distillation improves the student’s weighted F1; when the gap becomes negative, the benefit vanishes or reverses, in line with the CCH.

To extend from single-student distillation to multimodal learning, we compare two fusion strategies—direct fusion and fusion with knowledge distillation (Fusion+KD; Fig.~\ref{fig:fusion}). On KIPAN (Table~\ref{tab:fusion-kd-results}; additional results in Appendix~\ref{app:cancer}), mRNA as teacher yields $I_{\mathrm{TS}} > I_{\mathrm{SL}}$ and Fusion+KD outperforms direct fusion. In contrast, with RPPA as teacher we have $I_{\mathrm{TS}} < I_{\mathrm{SL}}$, and direct fusion is superior. These results suggest a practical design rule: incorporate KD in fusion only when $I_{\mathrm{TS}} > I_{\mathrm{SL}}$.

\begin{figure}[!t]
  \centering
  \begin{subfigure}[t]{0.48\textwidth}
    \centering
    \includegraphics[width=\linewidth]{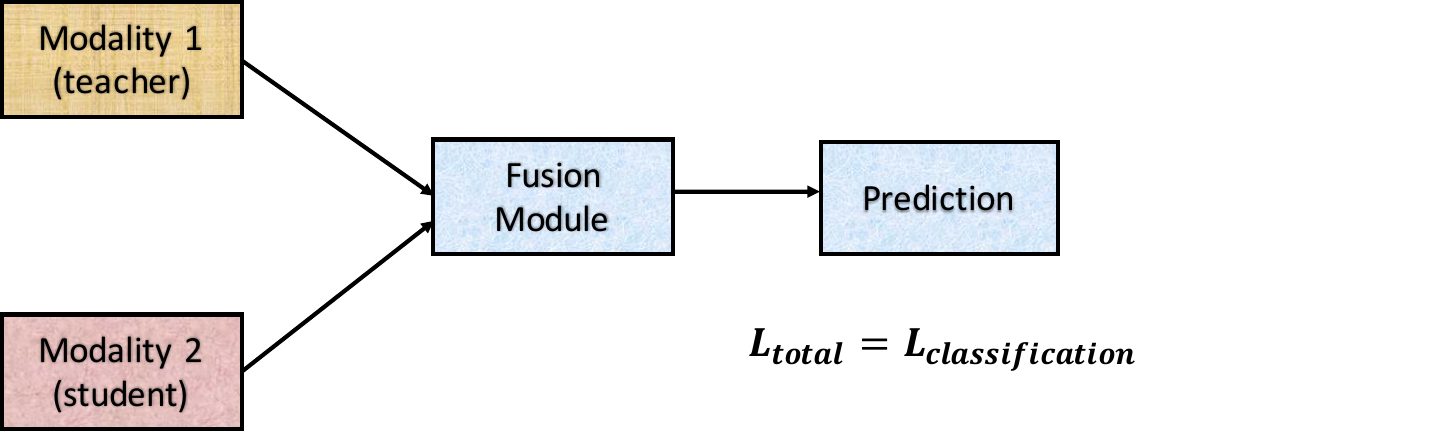}
    \caption{Direct fusion}
  \end{subfigure}\hfill
  \begin{subfigure}[t]{0.48\textwidth}
    \centering
    \includegraphics[width=\linewidth]{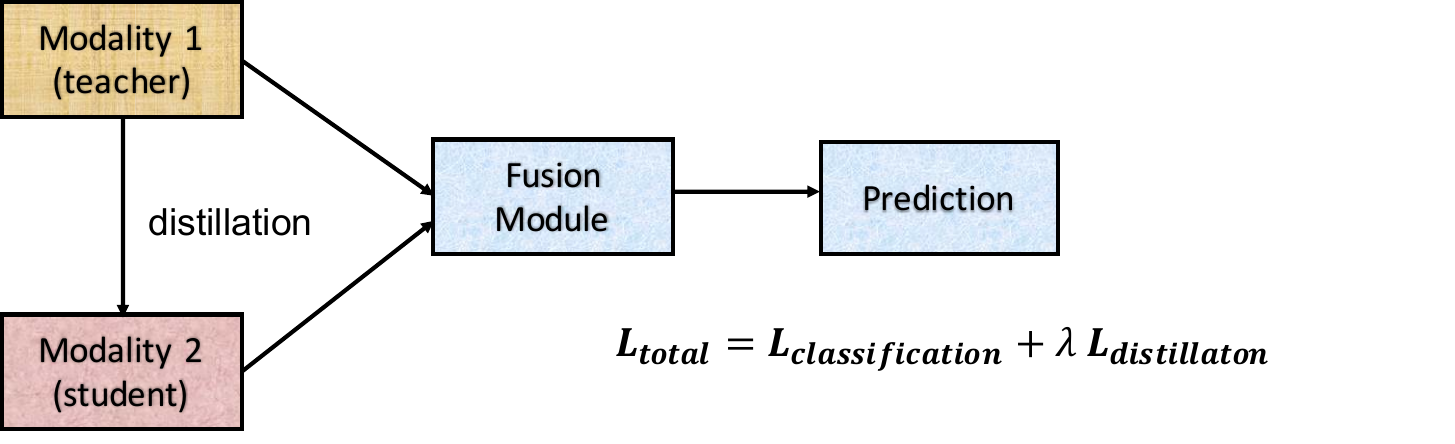}
    \caption{Fusion with KD}
  \end{subfigure}
  \caption{Multimodal fusion architectures: direct fusion (left) and Fusion+KD (right).}
  \label{fig:fusion}
\end{figure}

\begin{table}[!t]
  \centering
  \scriptsize
  \caption{Overall multimodal performance of direct fusion and Fusion+KD on KIPAN, reported with mutual information of modality representations (teacher–label, teacher–student, student–label).}
  \label{tab:fusion-kd-results}
  \setlength\tabcolsep{4pt}
  \resizebox{\textwidth}{!}{%
    \begin{tabular}{lccccccccccc}
      \toprule
       & \multicolumn{3}{c}{Mutual Information} 
       & \multicolumn{4}{c}{Fusion} 
       & \multicolumn{4}{c}{Fusion+KD} \\
      \cmidrule(lr){2-4}\cmidrule(lr){5-8}\cmidrule(lr){9-12}
       & Teacher–Label & Teacher–Student & Student–Label 
       & Acc & AUC & Macro F1 & Weighted F1 
       & Acc & AUC & Macro F1 & Weighted F1 \\
      \midrule
      mRNA (teacher)\\ CNV (student)
       & 1.0458 & 0.7898 & 0.6994 
       & 0.9610 & 0.9851 & 0.9219 & 0.9591 
       & 0.9740 & 0.9872 & 0.9293 & 0.9725 \\
        \hline
      RPPA (teacher)\\ CNV (student)
       & 1.1609 & 0.6893 & 0.6994 
       & 0.9740 & 0.9995 & 0.9333 & 0.9721 
       & 0.9610 & 0.9971 & 0.9225 & 0.9595 \\
      \bottomrule
    \end{tabular}%
  }
\end{table}

%% file: updates/Conclusion.tex
\section{Conclusion}
This paper introduced the Cross-modal Complementarity Hypothesis (CCH), a framework for explaining when cross-modal knowledge distillation (KD) improves performance in multimodal learning. The CCH offers a tractable, \emph{a priori} criterion for success: distillation is beneficial when the mutual information between teacher and student representations exceeds that between the student representation and the labels. We validated the hypothesis with a theoretical analysis in a joint Gaussian model and with experiments spanning synthetic Gaussian data and diverse real-world modalities—image, text, video, and audio—as well as three cancer omics datasets.

Our results highlight mutual information as a reliable predictor of cross-modal KD efficacy, yielding both theoretical insight and practical guidance for selecting teacher modalities to strengthen weaker ones. 

%% file: updates/Appendix.tex

\section{Theoretical analysis}
\label{app:theory}
Here we prove a more complete version of Theorem \ref{theo:main}.
\begin{theorem}
For $\kappa>1$ and almost every $\lambda$, there exists $w_1$ such that $R(\lambda,\tilde{w})<R(\lambda,0)$ asymptotically. Moreover, for $\lambda$ small enough, we have $R(\lambda,\tilde{w})<R_0$ asymptotically if $w_1^T\Sigma_{11}w_1\leq\Sigma_{33},w_1^T\Sigma_{13}\geq0$ and $I(w_1^Tx_1,(w^*)^Tx_2)>I((w^*)^Tx_2,y)$.
\label{theo:smaller_risk}
\end{theorem}
\begin{proof}
The optimization problem eq. (\ref{eq:old_loss}) is equivalent to
\begin{equation}
\hat{w}:=\arg\min_{w_2}\sum_{i=1}^n \Bigl\|\tilde{y}_i - w_2^T x_{2i}\Bigr\|^2,
\end{equation}
where the effective label is given by
\begin{equation}
\bar{y}_i:=\frac{1}{1+\lambda}(y_i+\lambda w_1^Tx_{1i}).
\end{equation}
It satisfies $\bar{y}_i=\bar{w}^Tx_{2i}+\mathcal{N}(0,\bar{\sigma}^2)$, where
\begin{equation}
\bar{w}:=\frac{1}{1+\lambda}
\Sigma_{22}^{-1}(\Sigma_{23}+\lambda\Sigma_{12}^Tw_1)
\end{equation}
and
\begin{equation}
\tilde{\sigma}^2:=\mathbb{E}[\bar{y}_n^2]-\bar{w}^T\Sigma_{22}\bar{w}.
\end{equation}
According to Theorem 3 of \cite{chang2021provable}, the estimator $\hat{w}$ can be expressed asymptotically as
\begin{equation}
\hat{w}=\bar{w}+\bar{\sigma}\frac{\Sigma_{22}^{-1/2}g}{\sqrt{p(\kappa-1)}},
\end{equation}
where $g\sim\mathcal{N}(0,I_p)$. Thus the asymptotics of $R(\lambda,w_1)$ is
\begin{equation}
\begin{aligned}
\bar{R}(\lambda,w_1)&=(\bar{w}-w^*)\Sigma_{22}(\bar{w}-w^*)+\tilde{\sigma}^2\frac{1}{\kappa-1}\\
&=\frac{\lambda^2}{(1+\lambda)^2}(
\Sigma_{22}^{-1}\Sigma_{12}^Tw_1-w^*)^T\Sigma_{22}(\Sigma_{22}^{-1}\Sigma_{12}^Tw_1-w^*)\\
&\quad+\frac{1}{\kappa-1}\frac{1}{(1+\lambda)^2}[\Sigma_{33}-(w^*)^T\Sigma_{22} w^*+2\lambda w_1^T(\Sigma_{13}-\Sigma_{12} w^*)+\lambda^2w_1^T(\Sigma_{11}-\Sigma_{12}\Sigma_{22}^{-1}\Sigma_{12}^T)w_1],
\end{aligned}
\end{equation}
where we denote $w^*=\Sigma_{22}^{-1}\Sigma_{23}$ to be the optimal weight. Here "asymptotics" means that $\lim_{n,p\to\infty}\mathbb{P}\left(\sup_{||w_1||<M}|R(\lambda,w_1)-\bar R(\lambda,w_1)|>\epsilon\right)=0$ for any $\epsilon>0$. Taking the derivative of $\bar{R}$ w.r.t. $w_1$, we have that the optimal $w_1$ is given by
\begin{equation}
\lambda\left[\Sigma_{12}^T\Sigma_{22}^{-1}\Sigma_{12}+\frac{1}{\kappa-1}(\Sigma_{11}-\Sigma_{12}\Sigma_{22}^{-1}\Sigma_{12}^T)\right]w_1=\lambda\Sigma_{12}^Tw^*-\frac{1}{\kappa-1}(\Sigma_{13}-\Sigma_{12} w^*).
\end{equation}
This gives an optimal $w_1$ for almost every $\lambda$. The optimal $w_1$ is non-zero and different from the optimal teacher weight $w^*$ for almost every $\lambda$. For the special case $\Sigma_{13}-\Sigma_{12}\Sigma_{22}^{-1}\Sigma_{23}=0$ (i.e. $x_1$ and $y$ are independent conditioned on $x_2$), the optimal surrogate weight is given by
\begin{equation}
w_1=(\kappa-1)(\Sigma_{11}+(\kappa-2)\Sigma_{12}\Sigma_{22}^{-1}\Sigma_{12}^T)^{-1}\Sigma_{12}^Tw^*,
\end{equation}
which does not depend on $\lambda$.

Moreover, for small $\lambda$, we have
\begin{equation}
\bar R(\lambda,w_1)=\frac{1}{\kappa-1}\frac{1}{(1+\lambda)^2}(\Sigma_{33}-(w^*)^T\Sigma_{22} w^*)+\frac{2\lambda}{\kappa-1}w_1^T(\Sigma_{13}-\Sigma_{12}w^*)+O(\lambda^2),
\end{equation}
and thus $\bar R(\lambda,w_1)<\bar{R}(0,w_1)$ for small $\lambda$ if 
\begin{equation}
\hat{w}^T(\Sigma_{13}-\Sigma_{12}\Sigma_{22}^{-1}\Sigma_{23})-(\Sigma_{33}-(w^*)^T\Sigma_{22} w^*)<0.
\label{eq:condition}
\end{equation}

Now we define the correlation between $w_1x_1$ and $w^*x_2$ to be 
\begin{equation}
\rho(w_1x_1,w^*x_2):=\frac{w_1^T\Sigma_{12}w^*}{\sqrt{w_1^T\Sigma_{11}w_1}\sqrt{(w^*)^T\Sigma_{22}w^*}}.
\end{equation} 
Similarly we define 
\begin{equation}
\rho(w_1x_1,y):=\frac{w_1^T\Sigma_{13}}{\sqrt{w_1^T\Sigma_{11}w_1}\sqrt{(w^*)^T\Sigma_{22}w^*}}
\end{equation}
and
\begin{equation}
\rho(w^*x_2,y):=\frac{(w^*)^T\Sigma_{23}}{\sqrt{(w^*)^T\Sigma_{22}w^*}\sqrt{\Sigma_{33}}}=\frac{\sqrt{(w^*)^T\Sigma_{22}w^*}}{\sqrt{\Sigma_{33}}}.
\end{equation}
Then the condition \eqref{eq:condition} becomes
\begin{equation}
\rho(w_1x_1,w^*x_2)>\frac{\rho(w_1x_1,y)}{\rho(w^*x_2,y)}-\frac{1-\rho(w^*x_2,y)^2}{\rho(w^*x_2,y)}\frac{\sqrt{\Sigma_{33}}}{\sqrt{w_1^T\Sigma_{11}w_1}}.
\end{equation}
Therefore, if $I(w_1^Tx_1,(w^*)^Tx_2)>I((w^*)^Tx_2,y)$ we have
\begin{equation}
\begin{aligned}
\rho(w_1x_1,w^*x_2)&>\rho(w^*x_2,y)=\frac{1}{\rho(w^*x_2,y)}-\frac{1-\rho(w^*x_2,y)^2}{\rho(w^*x_2,y)}\\&\geq\frac{\rho(w_1x_1,y)}{\rho(w^*x_2,y)}-\frac{1-\rho(w^*x_2,y)^2}{\rho(w^*x_2,y)}\frac{\sqrt{\Sigma_{33}}}{\sqrt{w_1^T\Sigma_{11}w_1}}.
\end{aligned}
\end{equation}
Thus the condition \eqref{eq:condition} is satisfied and we have $\bar{R}(\lambda,w_1)<\bar{R}(0,w_1)$. For the first inequality we use $I(A,B)=-\frac{1}{2}\log(1-\rho(A,B)^2)$ for Gaussian variables $A,B$ and the fact that $\rho(w^*x_2,y),\rho(w_1x_1,y)\geq0$ if $w_1^T\Sigma_{13}\geq0$. The last inequality is from $\rho(w_1x_1,y)\leq1$ and $\frac{\sqrt{\Sigma_{33}}}{\sqrt{w_1^T\Sigma_{11}w_1}}\leq1$. This finishes the proof.
\end{proof}

For completeness we also prove that knowledge distillation might help in the overparameterization regime.
\begin{theorem}
For $\kappa<1$ and almost every $\lambda$, there also exists $w_1$ such that $R(\lambda,w_1)<R(\lambda,0)$ asymptotically.
\end{theorem}
\begin{proof}
For $\kappa<1$ we are in the overparameterization case and thus we consider the minimal norm estimator
\begin{equation}
\hat{w}=\arg\min_w\left\{||w||:\sum_{i=1}^n||\frac{1}{1+\lambda}(y_i+\lambda w_1^Tx_{1i})-w^Tx_{2i}||^2=0\right\}.
\end{equation}
We can rewrite it as
\begin{equation}
\hat{w}=\frac{\bar{\sigma}}{\sigma}\arg\min_w\left\{||w||:\sum_{i=1}^n||\frac{\sigma}{\bar{\sigma}}\bar{y}_i-w^Tx_{2i}||^2=0\right\},
\end{equation}
where we recall that the effective label satisfies $\frac{\sigma}{\bar{\sigma}}\bar{y}_i=\frac{\sigma}{\bar{\sigma}}\bar{w}^Tx_{2i}+\mathcal{N}(0,\sigma^2)$.

Then we can use \cite[Theorem 4]{ildiz2024high} for the function $f(x)=||\Sigma_{22}^{1/2}(\frac{\bar{\sigma}}{\sigma}x-w^*)||^2$ to obtain the following asymptotic excess risk
\begin{equation}
 \begin{aligned}
\bar{R}(\lambda,w_1)=&(w_s-w^*)^T\theta_1^T\Sigma_{22}\theta_1(w_s-w^*)+\gamma(w^s)\mathbb{E}_{g_t}[\theta_2^T\Sigma_{22}\theta_2]\\&+w^*(I-\theta_1)^T\Sigma_{22}(I-\theta_1)w^*-2(w^*)^T(I-\theta_1)^T\Sigma_{22}\theta_1(w_s-w^*),
\end{aligned}
\end{equation}
where we denote $w_s:=\frac{\sigma}{\bar{\sigma}}\bar{w}$ and $\tau$ to be the solution of $\kappa=\frac{1}{p}\text{tr}((\Sigma_{22}+\tau I)^{-1}\Sigma_{22})$,
\begin{equation}
\theta_1:=\frac{\bar{\sigma}}{\sigma}(\Sigma_{22}+\tau I)^{-1}\Sigma_{22},\ \theta_2:=\frac{\bar{\sigma}}{\sigma}(\Sigma_{22}+\tau I)^{-1}\Sigma_{22}^{1/2}\frac{g_t}{\sqrt{p}},
\end{equation}
and $g_t\sim\mathcal{N}(0,I_p)$. Moreover, $\gamma(w_s)$ is given by
\begin{equation}
\gamma^2(w_s)=\kappa^{-1}\frac{\sigma^2+\tau^2||\Sigma_{22}^{1/2}(\Sigma_{22}+\tau I)^{-1}w_s||^2}{1-\frac{1}{n}\text{tr}((\Sigma_{22}+\tau I)^{-2}\Sigma_{22}^2)}.
\end{equation}
The results can be simplified to
\begin{equation}
\begin{aligned}
\bar{R}(\lambda,w_1)=&\frac{\bar{\sigma}^2}{\sigma^2}(w_s-w^*)\Sigma_{22}^3(\Sigma_{22}+\tau I)^{-2}(w_s-w^*)+\frac{\bar{\sigma}^2}{\sigma^2}\Omega\frac{\sigma^2+\tau^2||\Sigma_{22}^{1/2}(\Sigma_{22}+\tau I)^{-1}w_s||^2}{1-\Omega}\\
&-2\frac{\bar{\sigma}}{\sigma}(w^*)^T\Sigma_{22}^2(\Sigma_{22}+\tau I)^{-2}(\Sigma_{22}+\tau I-\frac{\bar{\sigma}}{\sigma}\Sigma_{22})(w_s-w^*)\\&+w^*(\Sigma_{22}+\tau I-\frac{\bar{\sigma}}{\sigma}\Sigma_{22})^2(\Sigma_{22}+\tau I)^{-2}\Sigma_{22}w^*,
\end{aligned}
\label{eq:bar_R}
\end{equation}
where we denote $\Omega:=\frac{1}{n}\text{tr}((\Sigma_{22}+\tau I)^{-2}\Sigma_{22}^2)$. Therefore, the optimal $w_1$ is given by the saddle points of \eqref{eq:bar_R}, where
 \begin{equation}
w_s:=\frac{\sigma}{(1+\lambda)\bar{\sigma}}(w^*+\lambda\Sigma_{22}^{-1}\Sigma_{12}^Tw_1)
\end{equation}
and
\begin{equation}
\bar{\sigma}:=\frac{1}{1+\lambda}\sqrt{\sigma^2+2\lambda w_1^T(\Sigma_{13}-\Sigma_{12} w^*)+\lambda^2w_1^T(\Sigma_{11}-\Sigma_{12}\Sigma_{22}^{-1}\Sigma_{12}^T)w_1}.
\end{equation}
\end{proof}

\section{Experimental details and results for synthetic data}\label{app:synthetic}

We evaluate the Cross-modal Complementarity Hypothesis (CCH) on a controlled synthetic regression benchmark. We generate \(n\) i.i.d.\ samples \(\{(X_{1,i}, X_{2,i}, Y_i)\}_{i=1}^n\) as follows:
\begin{align*}
Y_i &\sim \mathcal{N}(0,1),\\
X_{2,i}\mid Y_i &\sim \mathcal{N}\!\bigl(\sigma_{23} Y_i\,\mathbf{1}_p,\; (1-\sigma_{23}^2) I_p\bigr),\\
X_{1,i}\mid X_{2,i}, Y_i &\sim \mathcal{N}\!\bigl(a\,X_{2,i} + b\,Y_i,\; v\,I_p\bigr),
\end{align*}
where
\[
\phi = 1 - \sigma_{23}^2,\quad
a = \frac{\sigma_{12} - \sigma_{13}\sigma_{23}}{\phi},\quad
b = \frac{\sigma_{13} - \sigma_{12}\sigma_{23}}{\phi},\quad
v = 1 - \frac{\sigma_{12}^2 + \sigma_{13}^2 - 2\,\sigma_{12}\sigma_{13}\sigma_{23}}{\phi}.
\]

Both teacher and student use the fully connected architecture in Table~\ref{tab:synthetic_arch}. We train on \(10000\) samples and hold out \(1000\) for testing. Models are optimized with Adam (learning rate \(0.01\)) for \(300\) epochs; full settings appear in Table~\ref{tab:synthetic_param}.

\begin{figure}[!t]
  \centering
  \begin{subfigure}[b]{0.45\textwidth}
    \includegraphics[width=\linewidth]{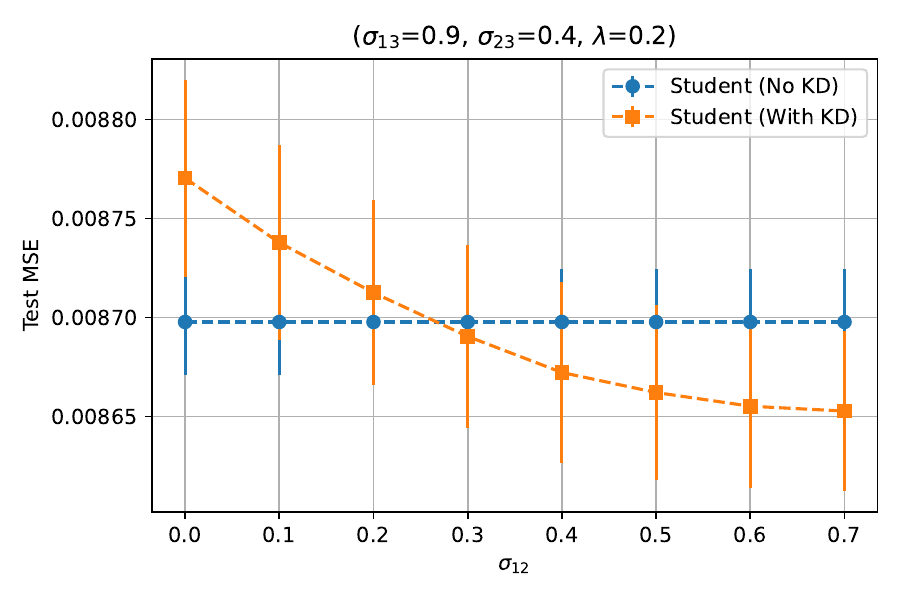}
    \caption{\(\lambda=0.2\)}
  \end{subfigure}
  \hfill
  \begin{subfigure}[b]{0.45\textwidth}
    \includegraphics[width=\linewidth]{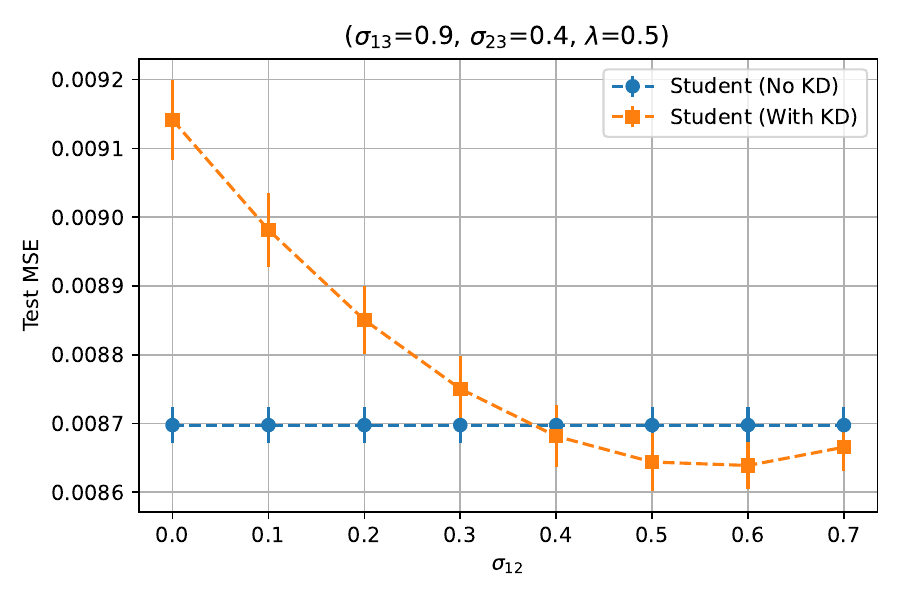}
    \caption{\(\lambda=0.5\)}
  \end{subfigure}

  \vspace{1ex}
  \begin{subfigure}[b]{0.45\textwidth}
    \includegraphics[width=\linewidth]{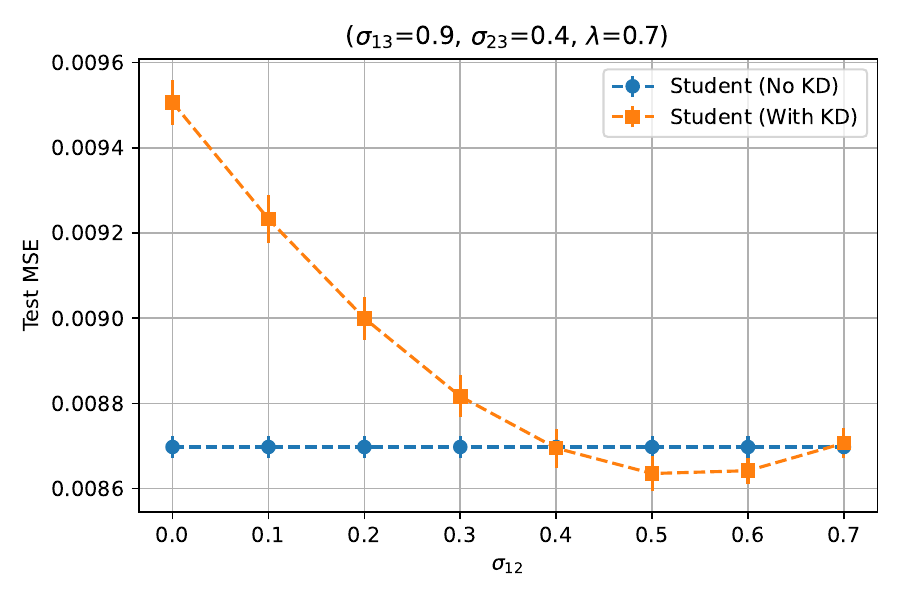}
    \caption{\(\lambda=0.7\)}
  \end{subfigure}
  \hfill
  \begin{subfigure}[b]{0.45\textwidth}
    \includegraphics[width=\linewidth]{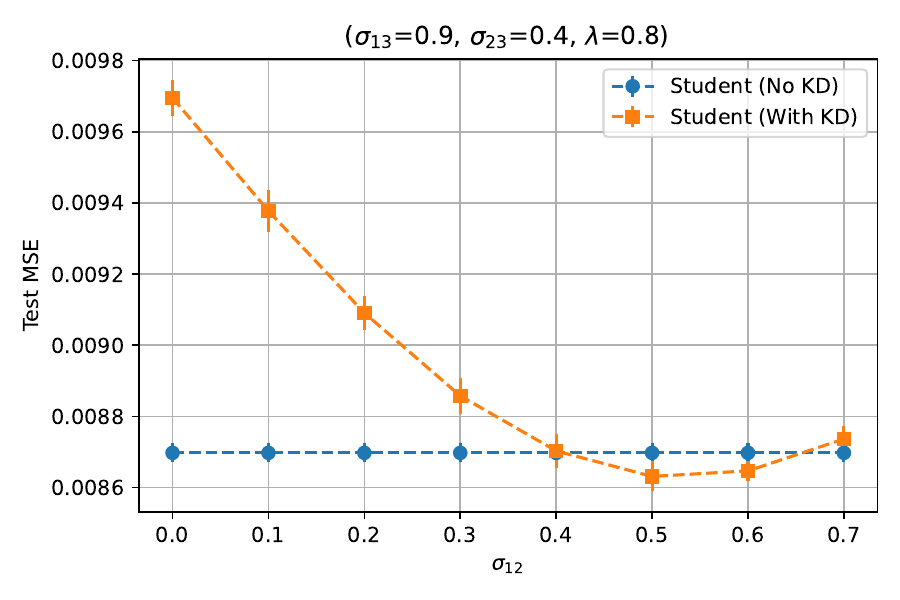}
    \caption{\(\lambda=0.8\)}
  \end{subfigure}

  \caption{Test MSE on synthetic regression data for varying distillation weight \(\lambda\). Orange dashed curves: student with KD; blue dashed curves: student without KD.}
  \label{fig:synthetic_lambda_grid}
\end{figure}

Figure~\ref{fig:synthetic_lambda_grid} reports test mean-squared error (MSE) as a function of the inter-modality correlation \(\sigma_{12}\) for distillation weights \(\lambda \in \{0.2, 0.5, 0.7, 0.8\}\). Because varying only \(\lambda\) does not change the learned representations’ mutual information (MI), the MI curves coincide with those obtained at \(\lambda=0.3\) (see Fig.~\ref{fig:synthetic_nonlinear_lambda05}). From Fig.~\ref{fig:synthetic_lambda_grid}, when \(\sigma_{12}\) is large (e.g., \(\sigma_{12}=0.7\), indicating strong teacher–student alignment), distillation improves the student provided two conditions hold: (i) the CCH criterion \(I(H_1;H_2)>I(H_2;Y)\) and (ii) a sufficiently small \(\lambda\) to avoid over-regularizing toward the teacher. This behavior is consistent with Theorem~\ref{theo:main}.

\begin{table}[!h]
\centering
\caption{Network architecture for synthetic experiments.}
\label{tab:synthetic_arch}
\begin{tabular}{lcc}
\toprule
\textbf{Layer} & \textbf{\# Units} & \textbf{Activation} \\
\midrule
Input & 100 & -- \\
Linear & 64 & ReLU \\
Linear & 1 & -- \\
\bottomrule
\end{tabular}
\end{table}

\begin{table}[!ht]
\centering
\caption{Training configuration and dataset details for synthetic experiments.}
\label{tab:synthetic_param}
\begin{tabular}{lc}
\toprule
\textbf{Item} & \textbf{Value} \\
\midrule
Training dataset & Synthetic Gaussian \\
Train/Test split & 10{,}000 / 5{,}000 \\
Optimizer & Adam \\
Learning rate & 0.01 \\
Epochs & 300 \\
\bottomrule
\end{tabular}
\end{table}

\section{Experimental details and results for image Data}\label{app:image}

We evaluate our approach using the MNIST~\citep{lecun1998gradient} and MNIST-M~\citep{ganin2015unsupervised} datasets. MNIST comprises 70,000 $28\times28$ grayscale images of handwritten digits (0–9). MNIST-M adapts these digits by blending them onto natural-image backgrounds sampled from the BSDS500 dataset~\citep{martin2001database}, resulting in colored images with identical labels (Figure~\ref{fig:mnist}). Below, we detail the MNIST-M construction, the network architecture, training configuration, and additional results for varying blending coefficients.

\begin{figure}[htbp]
  \centering
  \includegraphics[width=0.6\textwidth]{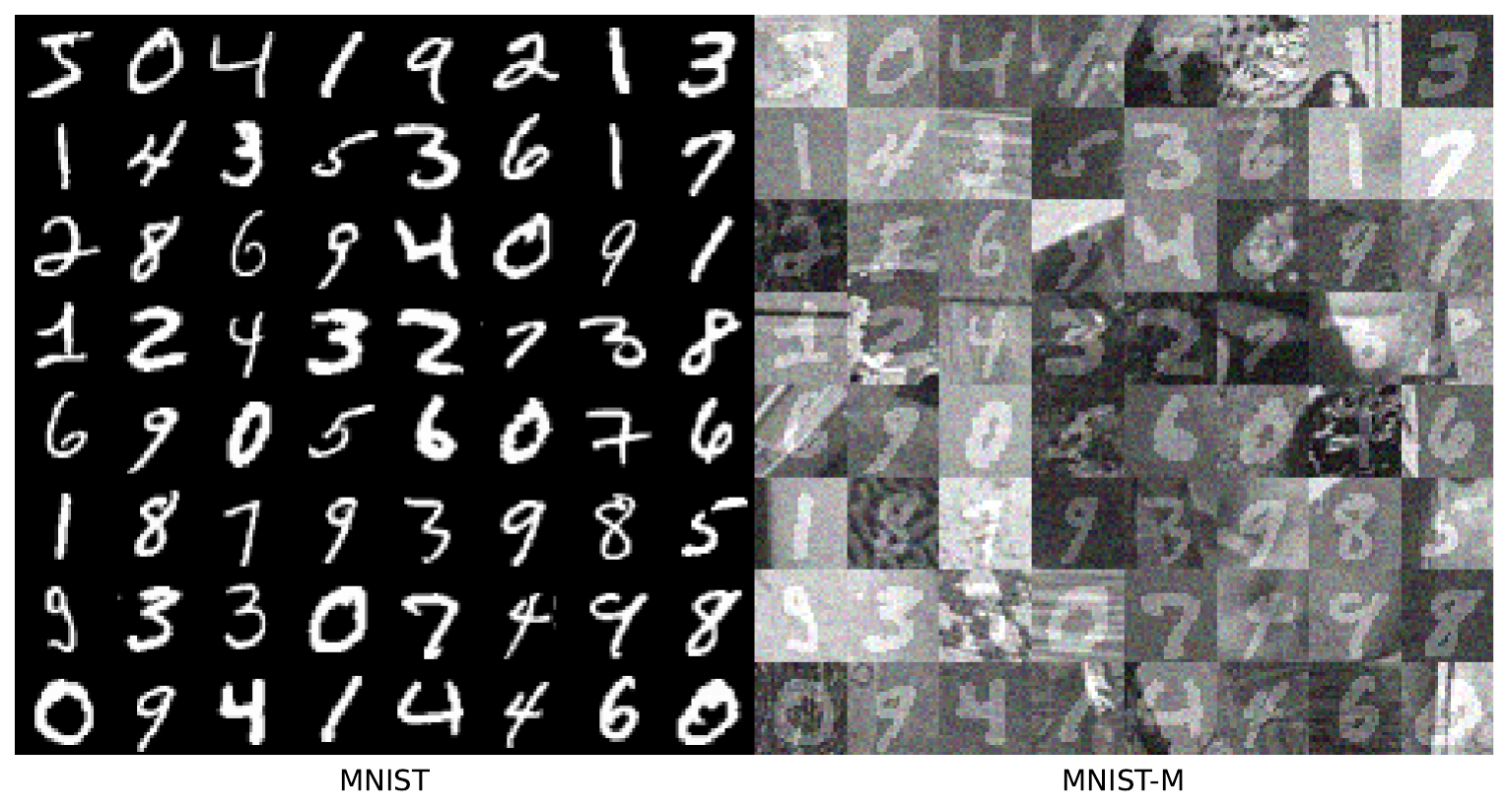}
  \caption{Sample images from MNIST (left) and MNIST-M (right).}
  \label{fig:mnist}
\end{figure}

\begin{algorithm}[ht]
\caption{Cross-modal knowledge distillation protocol for image data}
\label{alg:image_protocol}
\KwIn{MNIST and MNIST-M datasets}
\KwOut{Test accuracy of student with and without distillation}
1: \textbf{Teacher pretraining:} Train a teacher network on MNIST\;  
2: \textbf{Student baseline:} Train a student network on MNIST-M using only ground-truth labels\;  
3: \textbf{Distillation:}\;  
4: \quad Freeze teacher parameters\;  
5: \quad \For{each Gaussian blur level $\gamma$}{  
6: \quad\quad Apply Gaussian blur of intensity $\gamma$ to teacher inputs\;  
7: \quad\quad Obtain soft targets from the frozen teacher\;  
8: \quad\quad Train a new student on MNIST-M using both labels and soft targets (Eq.~\ref{eq:crosskd_classification})\;  
   }  
9: \textbf{Evaluation:} Evaluate both student models on the MNIST-M test set\;  
\end{algorithm}

To generate each MNIST-M image, we first binarize the original MNIST digit via thresholding and replicate the resulting single-channel image across the red, green, and blue channels, ensuring compatibility with RGB-based network architectures while preserving the digit’s grayscale silhouette. We apply a luminance-preserving transformation to convert BSDS500 patches to grayscale, matching the teacher modality. We then extract a random $28\times28$ patch $I_{\mathrm{BSDS}}$ from the processed BSDS500 images and compute:
\begin{equation*}
    I_{\mathrm{MNISTM}} = \alpha\,I_{\mathrm{MNIST}} + (1-\alpha)\,I_{\mathrm{BSDS}},
\end{equation*}
where $\alpha\in[0,1]$ controls the digit’s prominence over the background. Having specified the MNIST-M construction, we conduct training and evaluation according to Algorithm~\ref{alg:image_protocol}. For the experiments in Figure~\ref{fig:accuracy_mi_image} and Table~\ref{tab:accuracy_mi_diff_image}, we set $\alpha=0.2$.

Both teacher and student models share the architecture listed in Table~\ref{tab:append_network_image} and the training parameters in Table~\ref{tab:append_param_image}. We train using stochastic gradient descent (learning rate $0.002$, 100 epochs) with a distillation temperature of $T=3$ and a loss weight $\lambda=0.5$. All experiments were executed on an NVIDIA A100 GPU.

\begin{table}[htbp]
\centering
\caption{Network architecture for image experiments.}
\label{tab:append_network_image}
\begin{tabular}{lcc}
\hline
\textbf{Operation}            & \textbf{Size} & \textbf{Activation} \\
\hline
Input $\to$ Linear layer      & 1024          & LeakyReLU           \\
Linear layer                  & 256           & LeakyReLU           \\
Linear layer                  & 10            & --                  \\
\hline
\end{tabular}
\end{table}

\begin{table}[htbp]
\centering
\caption{Training configuration and dataset details for image experiments.}
\label{tab:append_param_image}
\begin{tabular}{lc}
\hline
\textbf{Training Dataset} & \textbf{MNIST / MNIST-M} \\
\hline
Train/Test Split          & 60000 / 10000             \\
\hline
Optimizer                 & SGD                       \\
Learning Rate             & 0.002                     \\
Epochs                    & 100                       \\
$T$                       & 3                         \\
$\lambda$                 & 0.5                       \\
\hline
\end{tabular}
\end{table}

Table~\ref{tab:mnist_mnistm_alph018} presents results for $\alpha=0.18$ under the same settings. First, the sign of the student accuracy difference (Student Acc Diff) precisely matches that of the mutual-information gap (MI GAP), thereby confirming the CCH. Second, compared to the $\alpha=0.2$ setting shown in Figure~\ref{fig:accuracy_mi_image}, the lower blending weight reduces the mutual information shared between the MNIST (teacher) and MNIST-M (student) modalities. This reduction in shared information corresponds to a diminished—sometimes negative—distillation gain, demonstrating that student performance declines as the teacher–student mutual information decreases.

\begin{table}[htbp]
\centering
\caption{Experimental results for the MNIST/MNIST-M dataset for $\alpha=0.18$. MNIST is the teacher modality and MNIST-M is the student modality. The teacher network achieves a test accuracy score of $0.9812 \pm 0.0003$ and $I(H_{\mathrm{teacher}};Y)=1.9095$.}
\resizebox{0.98\textwidth}{!}{%
\begin{tabular}{ccccccc}
\toprule
Gamma Level & $I(H_{\mathrm{teacher}};H_{\mathrm{student}})$ & $I(H_{\mathrm{student}};Y)$ & Student KD Acc & Student No-KD Acc & MI GAP & Student Acc Diff \\
\midrule
0   & 1.3956 & 1.2685 & $0.8484 \pm  0.0019$ & $0.8338 \pm 0.0034$ & 0.1271  & $0.0146 \pm 0.0052$  \\
0.5 & 1.2949 & 1.2685 & $0.8425 \pm 0.0042$ & $0.8338 \pm 0.0034$ & 0.0264 & $0.0087 \pm 0.0070$\\
1.5 & 1.2533 & 1.2685 & $0.8296 \pm 0.0017$ & $0.8338 \pm 0.0034$ & -0.0152 & $-0.0042 \pm 0.0034$\\
2.5 & 0.9472 & 1.2685 & $0.6216 \pm 0.0243$ & $0.8338 \pm 0.0034$ & -0.3213 & $-0.2122 \pm 0.0232$ \\
3.5 & 0.7817 & 1.2685 & $0.3325 \pm 0.0179$ & $0.8338 \pm 0.0034$ & -0.4868 & $-0.5013 \pm 0.0190$\\
\bottomrule
\end{tabular}%
}
\label{tab:mnist_mnistm_alph018}
\end{table}

\section{Experimental details for CMU-MOSEI Data}\label{app:mosei}

The CMU Multimodal Opinion Sentiment and Emotion Intensity (CMU-MOSEI) dataset contains 23,453 video segments annotated for sentiment and emotion. Each segment includes time-aligned transcriptions, audio, and visual data, providing three distinct modalities. Our preprocessing protocol for these modalities is detailed in the Algorithm~\ref{alg:mosei_protocol}.

\SetNlSty{}{}{:}
\SetKwInOut{Input}{Input}

\begin{algorithm}[htbp]
\caption{MOSEI Preprocessing Protocol}
\label{alg:mosei_protocol}
\Input{CMU-MOSEI utterance-level dataset: text; time-aligned audio \& visual feature streams.}

\textbf{Data \& splits:} Use the official train/validation/test partition. 

\textbf{Text:} Tokenize texts and map tokens to pretrained word embeddings; treat
\emph{one token = one timestep}. 

\ForEach{utterance $u$ in the dataset}{%
  \textbf{Temporal alignment:} Find the first non-padding token index $s$ in text$(u)$; slice \emph{text/audio/vision} to start at $s$ (text defines the time base). 

  \textbf{Length control:} For each modality, truncate to at most $L{=}50$ steps, then
  right-pad with zeros to exactly $L$. 
}

\textbf{Labels:} For classification, set $y{=}1$ if sentiment score $>\!0$, else $y{=}0$. 

\textbf{Batching:} Collate as \emph{(vision, audio, text, label)} to form shapes
$(B,L,D_v)$, $(B,L,D_a)$, $(B,L,D_t)$; labels $(B,1)$; here $D_v=713$, $D_a=74$ and $D_t=300$.

\end{algorithm}

The network architecture is identical for all three modalities and is specified in Table~\ref{tab:mosei_architecture}. The architecture includes a temporal mean-pooling layer, which operates as follows: for a given batch of sequences \(X\in\mathbb{R}^{B\times L\times D}\), the layer averages feature vectors across the time dimension \(L\) to produce an output \(Z\in\mathbb{R}^{B\times D}\), where:
\[
Z_{b,d} \;=\; \frac{1}{L}\sum_{l=1}^{L} X_{b,l,d}
\qquad (b=1,\dots,B;\; d=1,\dots,D).
\]

\begin{table}[htbp]
\centering
\caption{Network architecture for the CMU-MOSEI experiments.}
\begin{tabular}{lcc}
\toprule
\textbf{Operation} & \textbf{Size} & \textbf{Activation} \\
\midrule
Input $(B{\times}L{\times}D)$ $\to$ Temporal Mean-Pool $\to$ Flatten & $B{\times}L{\times}D \to B{\times}D$ & -- \\
Linear Layer & $D \to 256$ & ReLU \\
BatchNorm1d + Dropout $(p{=}0.3)$ & -- & -- \\
Linear Layer & $256 \to 128$ & ReLU \\
BatchNorm1d + Dropout $(p{=}0.3)$ & -- & -- \\
Linear Layer (Classifier Head) & $128 \to 2$ & -- \\
\bottomrule
\end{tabular}
\label{tab:mosei_architecture}
\end{table}

The training configuration details are consistent across all models and are summarized in Table~\ref{tab:mosei_parameters}.

\begin{table}[htbp]
\centering
\caption{Training configuration and dataset details for CMU-MOSEI experiments.}
\label{tab:mosei_parameters}
\begin{tabular}{l c}
\toprule
\textbf{Training Dataset} & \textbf{CMU-MOSEI} \\
\midrule
Train/Validation/Test Split & 70\% / 10\% / 20\% \\
Optimizer & AdamW \\
Learning Rate & 0.0005 \\
LR Schedule & CosineAnnealingLR ($T_{\max}=\,$epochs, $\eta_{\min}=0$) \\
Epochs & 100 \\
Temperature ($T$) & 4.5 \\
Distillation Weight ($\lambda$) & 0.5 \\
\bottomrule
\end{tabular}
\end{table}

\section{Experimental details and results for cancer data}
\label{app:cancer}

For cancer data, Table~\ref{tab:subtypes} summarizes the subtype distributions. For the experiments of Tables~\ref{tab:brca_mRNAcnv}–\ref{tab:lihc_mRNAcnv}, the teacher and student networks share the same architecture used in the synthetic data experiments (see Table~\ref{tab:synthetic_arch}). Table~\ref{tab:append_param_cancer} summarizes the training configurations and dataset splits for the three cancer cohorts.

\begin{table}[ht]
  \centering
  \small
  \caption{Subtype distribution for the BRCA, KIPAN, and LIHC cohorts.}
  \label{tab:subtypes}
  \begin{tabular}{lccc}
    \toprule
     & \textbf{BRCA} & \textbf{KIPAN} & \textbf{LIHC} \\
    \midrule
    \textbf{Subtypes} 
     & \makecell[l]{Normal-like: 44 \\ Basal-like: 129 \\ HER2-enriched: 49 \\ Luminal A: 338 \\ Luminal B: 267} 
     & \makecell[l]{KICH: 63 \\ KIRC: 492 \\ KIRP: 212} 
     & \makecell[l]{Blast-Like: 39 \\ CHOL-Like: 18 \\ Liver-Like: 113} \\
    \bottomrule
  \end{tabular}
\end{table}

\begin{table}[htbp]
\centering
\caption{Training configuration and dataset details for cross-modal distillation experiments on BRCA, LIHC cancer data.}
\label{tab:append_param_cancer}
\begin{tabular}{lc}
\hline
\textbf{Training Dataset} & BRCA, LIHC \\
\hline
Train/Test Split          & 90\% / 10\%     \\
Optimizer                 & Adam            \\
Learning Rate             & 0.01            \\
Epochs                    & 200             \\
Temperature ($T$)         & 2               \\
Distillation Weight ($\lambda$) & 0.5      \\
\hline
\end{tabular}
\end{table}

We evaluated two multimodal fusion strategies: direct fusion and fusion with knowledge distillation (Fusion\,+\,KD) (Table~\ref{tab:fusion-kd-results}). Both strategies adopt the architecture in Table~\ref{tab:append_architecture_cancerfusion}, which uses separate encoders for each modality followed by feature concatenation (see Figure~\ref{fig:fusion}); each encoder comprises 64 units. In the cross-modal distillation protocol (Tables~\ref{tab:brca_mRNAcnv}–\ref{tab:lihc_mRNAcnv}), we pretrained the teacher network on its modality and then used its soft targets to guide the student (Algorithm~\ref{alg:image_protocol}). By contrast, the fusion experiments train both encoders jointly—without teacher pretraining—while applying a distillation loss to transfer knowledge. Table~\ref{tab:append_param_cancerfusion} lists the corresponding training parameters.

\begin{table}[htbp]
\centering
\small
\caption{Layer-by-layer specification for multimodal fusion experiments on cancer data.}
\label{tab:append_architecture_cancerfusion}
\begin{tabularx}{\linewidth}{@{} l l X c X @{}}
\toprule
\textbf{Branch} & \textbf{Layer}   & \textbf{I/O}                                  & \textbf{Act.} & \textbf{Notes}                   \\
\midrule
\textbf{Modality 1} 
                & Linear           & $n_{\mathrm{inputMod1}}\!\rightarrow\!n_{\mathrm{enc}}$ & ReLU          & FC projection                   \\
                & BatchNorm1d      & $n_{\mathrm{enc}}\!\rightarrow\!n_{\mathrm{enc}}$        & —             & Normalization                   \\
                & Dropout          & $n_{\mathrm{enc}}$                                       & —             & $p=0.25$                        \\
\addlinespace
\textbf{Modality 2}
                & Linear           & $n_{\mathrm{inputMod2}}\!\rightarrow\!n_{\mathrm{enc}}$ & ReLU          & FC projection                   \\
                & BatchNorm1d      & $n_{\mathrm{enc}}\!\rightarrow\!n_{\mathrm{enc}}$        & —             & Normalization                   \\
                & Dropout          & $n_{\mathrm{enc}}$                                       & —             & $p=0.25$                        \\
\addlinespace
\textbf{Fusion \& Classification}
                & Concat           & $2\,n_{\mathrm{enc}}$                                    & —             & Merge embeddings                \\
                & Linear (fusion)  & $2\,n_{\mathrm{enc}}\!\rightarrow\!n_{\mathrm{classes}}$ & —             & Joint-feature logits            \\
                & Linear (modality)& $n_{\mathrm{enc}}\!\rightarrow\!n_{\mathrm{classes}}$    & —             & Modality-specific logits        \\
\bottomrule
\end{tabularx}
\end{table}

\begin{table}[htbp]
\centering
\caption{Training configuration and dataset details for multimodal fusion experiments on KIPAN data.}
\label{tab:append_param_cancerfusion}
\begin{tabular}{lc}
\hline
\textbf{Training Dataset} & KIPAN \\
\hline
Train/Test Split          & 90\% / 10\%     \\
Optimizer                 & Adam            \\
Learning Rate             & 0.007           \\
Epochs                    & 200             \\
Temperature ($T$)         & 1               \\
Distillation Weight ($\lambda$) & 0.5      \\
\hline
\end{tabular}
\end{table}

To demonstrate the generality of our approach beyond the KIPAN cohort, we also conducted experiments on BRCA data. Table~\ref{tab:fusion-kd-results-brca} reports the performance metrics for direct fusion and Fusion\,+\,KD, and Table~\ref{tab:append_param_cancerfusion_brca} lists the corresponding training settings. Across all teacher–student pairs, the mutual information between teacher and student representations consistently exceeds that between student representations and labels, and the Fusion+KD strategy outperforms direct fusion, thereby corroborating the CCH.

\begin{table}[htbp]
  \centering
  \scriptsize
  \caption{Overall multimodal performance of direct fusion and Fusion+KD on BRCA, reported with mutual information of modality representations (teacher–label, teacher–student, student–label).}
  \label{tab:fusion-kd-results-brca}
  \setlength\tabcolsep{4pt}
  \resizebox{\textwidth}{!}{%
    \begin{tabular}{lccccccccccc}
      \toprule
       & \multicolumn{3}{c}{Mutual Information} 
       & \multicolumn{4}{c}{Fusion} 
       & \multicolumn{4}{c}{Fusion+KD} \\
      \cmidrule(lr){2-4}\cmidrule(lr){5-8}\cmidrule(lr){9-12}
       & Teacher–Label & Teacher–Student & Student–Label 
       & Acc & AUC & Macro F1 & Weighted F1 
       & Acc & AUC & Macro F1 & Weighted F1 \\
      \midrule
      mRNA (teacher)\\ CNV (student)
       & 1.1081 & 0.5057 & 0.2757 
       & 0.7711 & 0.9157 & 0.6432 & 0.7563 
       & 0.8434 & 0.8610 & 0.6533 & 0.8225 \\
        \hline
      RPPA (teacher)\\ CNV (student)
       & 0.7328 & 0.3367 & 0.2757 
       & 0.5663 & 0.7844 & 0.5604 & 0.5715 
       & 0.6024 & 0.7929 & 0.5897 & 0.6103 \\
      \bottomrule
    \end{tabular}%
  }
\end{table}

\begin{table}[htbp]
\centering
\caption{Training configuration and dataset details for multimodal fusion experiments on BRCA.}
\label{tab:append_param_cancerfusion_brca}
\begin{tabular}{lc}
\hline
\textbf{Training Dataset} & MNIST / MNIST-M \\
\hline
Train/Test Split          & 90\% / 10\%     \\
Optimizer                 & Adam            \\
Learning Rate             & 0.04           \\
Epochs                    & 200             \\
Temperature ($T$)         & 4               \\
Distillation Weight ($\lambda$) & 0.5      \\
\hline
\end{tabular}
\end{table}

\section{Methods for mutual information estimation}
\label{app:mi_estimators}

Mutual information quantifies the dependency between random variables, but its estimation remains challenging, especially when the underlying probability distributions are unknown. Exact mutual information computation is tractable only for small datasets with known distributions. To address this limitation, \citet{ref:kraskov2004estimating} introduced a k-nearest neighbors (kNN) estimator for mutual information between continuous random variables. This estimator was further extended by \citet{ref:ross2014mutual} to handle cases where one variable is discrete and the other continuous---a critical adaptation given that many real-world datasets involve mixed data types. More recent approaches, such as Mutual Information Neural Estimation (MINE) \citep{ref:belghazi2018mutual}, leverage neural networks to estimate mutual information between high-dimensional continuous variables. Additionally, a novel method known as latent mututal information (LMI) has been developed \citep{ref:gowri2024approximating}, which applies a nonparametric mutual information estimator to low-dimensional representations extracted by a theoretically motivated model architecture.